\documentclass[]{interact}

\usepackage{epstopdf}
\usepackage[caption=false]{subfig}
\usepackage{caption}
\usepackage{booktabs}
\usepackage{graphicx}
\usepackage[short]{optidef}
\usepackage{comment}
\usepackage{multirow}

\usepackage[natbibapa,nodoi]{apacite}
\theoremstyle{plain}
\newtheorem{theorem}{Theorem}[section]

\theoremstyle{definition}

\theoremstyle{remark}

\usepackage{amsmath}               
  {
      \theoremstyle{plain}
      \newtheorem{assumption}{Assumption}
  }

\usepackage{algorithm,algcompatible,amsmath}
\usepackage{algorithm}
\usepackage{algpseudocode}

\begin{document}

\title{Safe and Efficient Manoeuvring for Emergency Vehicles in Autonomous Traffic using Multi-Agent Proximal Policy Optimisation}

\author{
\name{Leandro Parada \textsuperscript{a,1} \thanks{CONTACT Leandro Parada. Email: l.parada-pradenas20@ic.ac.uk} \thanks{\textsuperscript{1}  These authors contributed equally.}, Eduardo Candela \textsuperscript{a,1}, Luis Marques \textsuperscript{a},  Panagiotis Angeloudis \textsuperscript{a}}
\affil{\textsuperscript{a} Centre for Transport Studies, Department of Civil and Environmental Engineering, Imperial College London, London SW7 2AZ, United Kingdom}
}

\maketitle

\begin{abstract}

Manoeuvring in the presence of emergency vehicles is still a major issue for vehicle autonomy systems. Most studies that address this topic are based on rule-based methods, which cannot cover all possible scenarios that can take place in autonomous traffic. Multi-Agent Proximal Policy Optimisation (MAPPO) has recently emerged as a powerful method for autonomous systems because it allows for training in thousands of different situations. In this study, we present an approach based on MAPPO to guarantee the safe and efficient manoeuvring of autonomous vehicles in the presence of an emergency vehicle. We introduce a risk metric that summarises the potential risk of collision in a single index. The proposed method generates cooperative policies allowing the emergency vehicle to go at $15 \%$ higher average speed while maintaining high safety distances. Moreover, we explore the trade-off between safety and traffic efficiency and assess the performance in a competitive scenario. 

\end{abstract}

\begin{keywords}
Deep Reinforcement Learning,
Multi-Agent Reinforcement Learning,
Autonomous Traffic,
Autonomous Vehicles,
Emergency Vehicles,
Cooperative overtaking
\end{keywords}

\section{Introduction}

Autonomous Vehicles (AVs) can potentially help to reduce the response time and the accidents involving Emergency Vehicles (EMVs). However, many challenges related to the control of autonomous vehicles in traffic still remain. One of these major challenges is the safe and efficient manoeuvring in the presence of emergency vehicles. The interaction between AVs and EMVs in traffic generates unusual and complex multi-agent scenarios that are difficult to train for current vehicle autonomy systems. 


EMVs move at high speeds, constantly overtaking other vehicles and passing through busy intersections. Moreover, ambulances and other EMVs can move in contraflow through the opposite lane in some countries, such as the UK. Those mentioned above are why many accidents involve EMVs, several of which have fatal consequences. 
In 2020, 180 people died in the US in collisions involving EMVs. Collisions with police vehicles accounted for 132 deaths, 31 were related to ambulances, and 17 to fire trucks. Of these deaths, 69\% were due to multi-vehicle crashes \citep{nsc_injury_ev}. These numbers have been steadily increasing during the past years, as roads are becoming increasingly congested. In the UK, police-related road accidents reached a 13-year high in 2019 \citep{gayle_number_2019}. EMV accidents are also costly, both for repairs and legal services. In London, ambulance crashes cost more than $\pounds 2$ million a year, considering only repair and insurance costs, according to a Freedom of Information Act request \citep{noauthor_ambulance_2017}.

Manoeuvring in the presence of emergency vehicles is still a major challenge for current vehicle autonomy systems (SAE Automation Level 2 \citep{noauthor_j3016c_nodate}). Several incidents have been reported in the last 5 years related to autonomy systems in urban settings. For instance, in May 2022 an autonomous car blocked a Fire Truck that was responding to an emergency, which caused property damage and personal injuries \citep{marshall_autonomous_nodate}. Moreover, vehicles using the Tesla Autopilot system have been repetitively reported to crash against EMVs \citep{business_tesla_nodate}, which have resulted in several injuries and one death. Therefore, there is a need to develop algorithms that can guarantee the safe manoeuvring around EMVs. 


A few studies have proposed methods for efficient overtaking for EMVs in autonomous traffic systems. The majority of these studies focus on determining the vehicle-passing sequence that minimises the delay of EMVs \citep{humayun_autonomous_2022, nellore_traffic_2016, lu_genetic_2017}. A recent study \citep{lu_genetic_2017}  proposed a solution method based on a Genetic Algorithm to find the optimal passing sequence of EMVs in an intersection where all vehicles are connected autonomous vehicles. Using a slightly different approach, a study proposed a game-theoretic control method for efficient overtaking manoeuvres for autonomous EMVs in non-autonomous traffic \citep{buckman_semi-cooperative_2021}. The method considers both safety and efficiency and includes a behavioural model of the human driver's willingness to cooperate. A conceptual model for routing and scheduling EMVs in the context of smart cities has also been presented \citep{humagain_routing_2019}. These methods are based on heavy assumptions of the environment, such as a set of predefined routes, or rule-based vehicle movement \citep{lu_genetic_2017, nellore_traffic_2016, humayun_autonomous_2022}.

Multi-agent Deep Reinforcement Learning (MARL) has emerged as a promising approach for the development of multi-AV driving policies \citep{zhou_smarts_2020}. MARL policies are represented using Deep Neural Networks (DNN), which can model complex Autonomous Driving behaviour that would otherwise require hand-crafted features or complicated rule-based methods. In addition, MARL allows for model-free learning---the agents can learn based on their experience; therefore, no underlying model of the process is required. This feature is especially convenient for multi-agent problems, where the interaction between agents and the evolution of the environment can be challenging to model. MARL has been successfully applied to different applications in transportation \citep{farazi_deep_nodate, haydari_deep_2022}, although the applications in multi-AV environments are still scarce \citep{palanisamy_multi-agent_nodate, zhou_smarts_2020, shalev-shwartz_safe_2016}. 

The overtaking of EMVs in an autonomous traffic system is fundamentally a multi-agent problem. To guarantee a safe and efficient passage of EMVs, all vehicle movements (including the EMV) should be optimised in real-time. To the best of the authors' knowledge, only two studies \citep{dresner_human-usable_nodate-1, gonzalez_autonomous_2021} have approached the EMV overtaking problem as a multi-agent problem. However, these studies are based on the assumption that a fixed set of rules governs the interaction between agents. Rule-based methods cannot possibly cover all situations that can occur in real-world traffic. Moreover, they cannot represent coordination between agents and adapt to changing environments.

In this study, we propose a method based on MARL to design safe and efficient policies for manoeuvring in the presence of EMVs. Specifically, we use the Multi-Agent Proximal Policy Optimisation (MAPPO) algorithm to generate real-time actions for all vehicles. MAPPO has recently shown to be superior to other MARL algorithms in complex cooperative environments \citep{yu_surprising_2021}. We first formulate the overtaking of EMVs in autonomous traffic as a stochastic game, where each agent can take independent decisions based on partial observation of the environment. In order to guarantee safe manoeuvring by both the EMV and the AVs, we introduce safety metrics and use them as an objective for each agent. We then develop a traffic environment consisting of several AVs and one EMV driving through heavily congested roads/highways. We compare the proposed method with two baseline microscopic traffic models, highlighting the benefit of using a multi-agent learning algorithm. Moreover, we assess the trade-off between safety and traffic efficiency over the reward function. We show that the proposed method is scalable and can guarantee the safe passage of EMVs under heavy traffic conditions and competitive scenarios.

The contributions of this paper can be summarised as follows: 

\begin{itemize}

    \item This paper is the first to propose an approach for strategic manoeuvring in the presence of EMVs based on the multi-agent learning paradigm. Previous approaches are based on rule-based vehicle behaviour.
    
    \item It provides a new risk modelling approach based on a single risk metric that allows to avoid collisions and maintain safe distances between all vehicles in traffic.
    
    \item It provides a analysis on the trade-off between traffic efficiency (represented as the average speed) and safety (represented by the unifying risk metric).
    
    
\end{itemize}

We recognise that deploying fully autonomous EMVs will pose additional technical, regulatory and even ethical challenges. Nevertheless, it is crucial to improve how current AV technologies handle scenarios involving EMVs. Therefore, we focus on developing risk-aware AV controllers at the strategic decision-making level \citep{pannocchi_integrated_2019}, while paving the way for autonomous EMV controllers, which will be required in the future. 


The rest of the paper is organised as follows. Section 2 summarises the relevant literature on the control of EMVs and autonomous traffic. Section 3 contains the risk modelling approach. Section 4 contains the MARL formulation. Section 5 presents the main results of the different case studies. Finally, section 6 presents the conclusion along with insights into future work.

\begin{figure}[t!]
  \centering
  \includegraphics[scale=0.7]{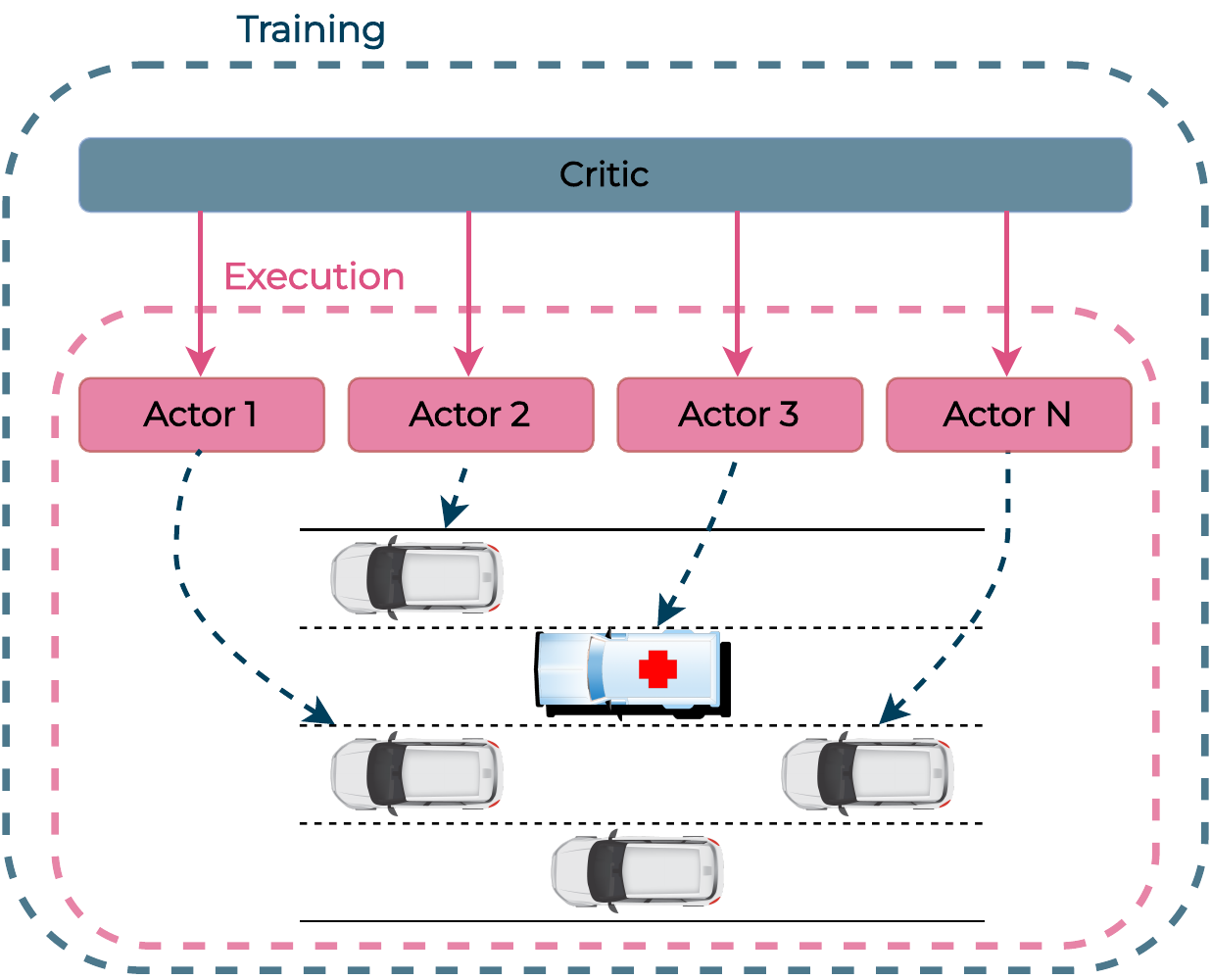}
  \caption{Graphic representation of the multi-agent autonomous traffic approach. Each AV as well as the EMV correspond to an actor network. During the training phase, all agents share a centralised critic. During execution, the critic is no longer needed.}
  \label{fig:approach}
\end{figure}

\section{Related Work}

In this section, we present the relevant literature on the control of EMVs and autonomous traffic systems. We classify the works into four categories: Intelligent Signal Control (ITS), Autonomous Traffic and, Risk-aware Autonomous Systems and Multi-Agent Reinforcement Learning (MARL).

\subsection{Intelligent Signal Control in presence of EMVs}
Extensive work has been done regarding Intelligent Traffic Signal (ITS) control in the presence of EMVs, i.e., determining the best green light sequence. A study \citep{chowdhury_priority_2016} presented an ITS method to optimise the travel time of EMVs that considers a priority system based on the type of incident. Similarly, a recent work \citep{karmakar_smart_2021} used a priority-based system to estimate the number of signal interventions while considering the impact on traffic in the surrounding areas. Virtual Traffic Lights (VTL) have also been used for efficient crossing of EMVs in autonomous traffic intersections \citep{viriyasitavat_priority_2012}. In VTL, an AV acts as a leader and manages the traffic by serving as a virtual traffic light for other vehicles. Another study \citep{shelke_fuzzy_2019} combined a message-passing approach with a priority-based traffic signal method that minimises the delay of data packages. A method for scheduling both green light sequence and duration was proposed by  \citep{nellore_traffic_2016}. Their method is based on sensing information, such as the distance to the intersection and the number of cars on the road. Two articles \citep{so_automated_2020, su_emvlight_2021} have provided joint methods to optimise the traffic signals and the vehicle movements. The latter couples a Decentralised Reinforcement Learning method to control the traffic signals with a dynamic Dijkstra's algorithm to find the optimal routes \citep{su_emvlight_2021}.

\subsection{Autonomous Traffic}

Several studies have been conducted on optimising autonomous driving in traffic, particularly in highway scenarios \citep{claussmann_review_2020}. These studies have focused almost exclusively on single-agent control, where traffic is considered part of the environment. A study proposed a game-theoretic traffic model and evaluated two single-agent policies based on Stackelberg games and Decision Trees \citep{li_game-theoretic_2016}. Deep Reinforcement Learning has also been applied to generate high-level policies for the ego vehicle  \citep{nageshrao_autonomous_2019}. A similar study \citep{ngai_multiple-goal_2011} presented a Reinforcement Learning-based method that considers seven different autonomous driving goals in order to guarantee safe and efficient overtaking manoeuvres. The Q-learning algorithm has also been used to develop efficient single-agent lane-changing policies \citep{wang_reinforcement_2018}. A recent study proposed a multi-agent game-theoretic planning approach for competitive racing scenarios \citep{wang_game-theoretic_2021}. In this approach, the agent has to collaborate to avoid collisions, but the overall goal is to beat its competitors. The study used an algorithm that solves for the approximate Nash Equilibrium. It assumed that the information of other vehicles' intentions and objective functions were completely known to every other vehicle.

Other studies have focused on the control of autonomous traffic in intersections \citep{lu_genetic_2017, zhu_linear_2015, dai_quality--experience-oriented_2016, kamal_vehicle-intersection_2015, miculescu_polling-systems-based_2016,zhang_optimal_2016, ilgin_guler_using_2014}. These studies use different methods to determine the optimal passing sequence of AVs in intersections without traffic signals. The methods include Linear Programming \citep{zhu_linear_2015}, Genetic Algorithms \citep{lu_genetic_2017}, Polling Systems \citep{miculescu_polling-systems-based_2016} and Decentralized Optimal Control \citep{zhang_optimal_2016}. 
\subsection{Risk-aware Autonomous Systems}

Only limited research has been conducted on risk-aware AV controllers. A recent study \citep{zhang2021safe} used the soft actor-critic (SAC) algorithm and control theory-based Lyapunov functions to train an agent subject to safety constraints. The authors in \citep{wen2020safe} proposed Parallel Constrained Policy Optimization (PCPO), which consists of three neural networks to estimate the policy function, value function, and a risk function. They use constraints and lack explicit risk models that capture risk when constraints are not violated. In another study \citep{nyberg2021risk}, risk metrics were coupled with a sampling-based trajectory planner. Trajectory planners require a probabilistic model to sample from, which is complex to obtain and can be overcome with model-free RL approaches. An RL method introduced hard constraints for Autonomous Driving \citep{li_game-theoretic_2016}, where cars are not allowed to take certain actions, e.g., change lane when there's a car on the other lane. Similarly, another study \citep{nageshrao_autonomous_2019} introduced an explicit short-horizon safety check to avoid making catastrophic decisions. In this study, we use state-of-the-art risk metrics for AVs to avoid collisions and maintain safe distances. 

The field of RL that pertains to learning policies where safety is a critical objective is called Safe RL \citep{garcia2015comprehensive}. \cite{garcia2015comprehensive} segment existing literature into two categories: modification of the optimisation criteria or the exploration process. In some approaches, risk functions are learned during training, such as in \citep{geibel2005risk, garcia2012safe}. To the best of our knowledge, no work has been carried out on Safe Multi-agent RL for AVs with modification of the optimisation criterion using explicit risk functions.

\subsection{Multi-agent Reinforcement Learning for Autonomous Traffic systems}

A couple of studies have introduced Multi-Agent autonomous driving gym environments for the purpose of evaluating learning algorithms.  \cite{palanisamy_multi-agent_nodate} presented a Multi-Agent Connected Autonomous Driving Gym environment (MACAD) that allows the training of different Deep RL algorithms on an intersection scenario. The environment contains different features, such as the nature of agents (homogeneous/non-homogeneous), the nature of tasks (cooperative/competitive), and the nature of observations (full/partial). Similarly, \cite{zhou_smarts_2020} introduced a multi-agent gym environment for AVs that considers different traffic situations, such as lane merging, intersections, overtaking and roundabouts. It incorporates different controllers and baselines, as well as several evaluation metrics. 

The use of MARL has been explored for different tasks in autonomous traffic.
\cite{shalev-shwartz_safe_2016} approached the double-merging problem using a Policy Gradient method. Cooperative lane-changing manoeuvres in the presence of Human-Driven Vehicles (HDVs) have also been explored \citep{zhou_multi-agent_2022}. Similarly,  \cite{chen_deep_2022} proposed a MARL method for highway ramp-merging that considers the presence of HDVs was explored. \cite{dong_drl-based_2020} presented a combination of Graph Convolutional Neural Networks (GCN) with the Deep Q Network (DQN) algorithm for safe and cooperative lane-changing decisions on highways. Despite these varied applications, MARL has not been applied for the safe manoeuvring of autonomous traffic in the presence of EMVs.

\begin{table*}
\centering
\caption{Summary of literature on the passing of EMVs in traffic.}
\label{lit_review}
\resizebox{\textwidth}{!}{%
\begin{tabular}{@{}llccl@{}}
\toprule
Author & Approach & Autonomous Traffic & Rule-based & Scenario
\\ \hline
\cite{lu_genetic_2017} & Genetic Algorithm & \checkmark & x & Junction \\ 
\cite{humayun_autonomous_2022} & Priority rules & \checkmark & \checkmark & Junction 
\\
\cite{nellore_traffic_2016} & Distance-based algorithm & x & \checkmark & Junction \\
\cite{buckman_semi-cooperative_2021} & Game-theoretic & \checkmark & x & Highway \\
\cite{karmakar_smart_2021} & Priority rules & x & \checkmark & Junction \\
\cite{shelke_fuzzy_2019} & Priority ITS & x & \checkmark & Junction \\
\cite{dresner_human-usable_nodate-1} & Reservation-based & \checkmark & \checkmark & Junction\\
 \cite{gonzalez_autonomous_2021} & Priority rules & \checkmark & \checkmark & Junction

 \\

 \bottomrule
\end{tabular}}
\end{table*}

\section{Modelling Risk}

The deployment of AVs on roads poses significant safety risks. Moreover, training and validating AV controllers requires exposure to complex hazardous scenarios involving humans. Hence the necessity for risk-aware AV controllers that can better predict and avoid dangerous situations. 

We introduce a single risk index $r\in[0,1]$ that represents the likelihood of a collision happening between any two vehicles moving on a two-dimensional surface. This index is the product of a longitudinal risk index $r_{lon}\in[0,1]$ and a lateral risk index $r_{lat}\in[0,1]$, which are piecewise linear functions that depend on the definitions of safe longitudinal and lateral distances , respectively \citep{shalev2017formal}. The minimum safe longitudinal distance is

\begin{eqnarray}
	d^{lon}_{min} = \left[ v_{r}\rho + \frac{1}{2}\rho^{2}a_{max} + \frac{(v_{r} 
	+ \rho a_{max})^{2}}{2b_{min}} - \frac{v_{f}^{2}}{2b_{max}} \right]_{+},
\end{eqnarray}

where $[x]_{+}:=\text{max}\{x,0\}$; $c_{r}$ is a car with velocity $v_{r}$ that drives behind another car $c_{f}$ (in the same direction) with velocity $v_{f}$. For any braking of $c_{f}$ of at most $b_{max}$, $c_{r}$ has a response time of $\rho$ during which it accelerates by at most $a_{max}$, and immediately after the response starts braking by at least $b_{min}$ until necessary. Notice that the closer the two vehicles are to colliding, $c_{r}$ will have to brake harder to avoid crashing, until its brakes reach their maximum capability (which we denote $B$) and a collision is unavoidable. When replacing $b_{min}$ with $B$ in $d^{lon}_{min}$, a tighter bound is defined for the safe longitudinal distance as

\begin{eqnarray}
	d^{lon}_{min,brake} = \left[ v_{r}\rho + \frac{1}{2}\rho^{2}a_{max} + \frac{(v_{r} 
	+ \rho a_{max})^{2}}{2B} - \frac{v_{f}^{2}}{2b_{max}} 
	\right]_{+}.
\end{eqnarray}

Let $d^{lon}$ be the current longitudinal distance between cars. We define the longitudinal risk index

\begin{eqnarray}
  r_{lon}=\left\{
  \begin{array}{@{}ll@{}}
    0, & \text{if}\ d^{lon} \geq d^{lon}_{min} > 0 \\
    1-\frac{d^{lon} - d^{lon}_{min,brake}}{d^{lon}_{min} - d^{lon}_{min,brake}}, & \text{if}\ d^{lon}_{min} \geq d^{lon} \geq d^{lon}_{min,brake} > 0 \\
    1, & \text{otherwise} \\
  \end{array}\right.\qquad.
\end{eqnarray}

Similarly, for the lateral analysis, the safe lateral distance is 

\begin{eqnarray}
	d^{lat}_{min} = \left[ \frac{v_{left} + v_{left,\rho}}{2}\rho 
	+ \frac{v^{2}_{left,\rho}}{2b^{lat}_{min}}
	- \left(\frac{v_{right} + v_{right,\rho}}{2}\rho 
	- \frac{v^{2}_{right,\rho}}{2b^{lat}_{min}} \right)
	\right]_{+},
\end{eqnarray}
where w.l.o.g. car $c_{left}$ is to the left of $c_{right}$, $v_{left}$ and $v_{right}$ are their respective lateral velocities, $\rho$ is the same response time as in the longitudinal case, during which the two cars will apply a maximum lateral acceleration of $a^{lat}_{max}$ toward each other, and after that both cars will apply lateral braking of at least $b^{lat}_{min}$ until necessary; $v_{left,\rho} = v_{left} + \rho a^{lat}_{max}$ and $v_{right,\rho} = v_{right} - \rho a^{lat}_{max}$. As with the safe longitudinal distance, when replacing $b^{lat}_{min}$ with $B^{lat}$ (the maximum capable lateral braking of the vehicles) in $d^{lat}_{min}$, a tighter bound is defined for the safe lateral distance as 

\begin{eqnarray}
	d^{lat}_{min,brake} = \left[ \frac{v_{left} + v_{left,\rho}}{2}\rho 
	+ \frac{v^{2}_{left,\rho}}{2B^{lat}}
	- \left(\frac{v_{right} + v_{right,\rho}}{2}\rho 
	- \frac{v^{2}_{right,\rho}}{2B^{lat}} \right)
	\right]_{+}.
\end{eqnarray}

Let $d^{lat}$ be the current lateral distance between cars. We define the lateral risk index

\begin{eqnarray}
  r_{lat}=\left\{
  \begin{array}{@{}ll@{}}
    0, & \text{if}\ d^{lat} \geq d^{lat}_{min} > 0 \\
    1-\frac{d^{lat} - d^{lat}_{min,brake}}{d^{lat}_{min} - d^{lat}_{min,brake}}, & \text{if}\ d^{lat}_{min} \geq d^{lat} \geq d^{lat}_{min,brake} > 0 \\
    1, & \text{otherwise} \\
  \end{array}\right.\qquad.
\end{eqnarray}
When multiplying the lateral and longitudinal risk indices, and adding risk propensity parameters $\beta$, $\gamma > 0$, the unified risk index is

\begin{eqnarray}
	r = \left(r_{lon}\right)^{\beta}\left(r_{lat}\right)^{\gamma},
\end{eqnarray}

which clearly lies in the range $[0,1]$, and only takes a value $>0$ when both longitudinal and lateral risk indices are also $>0$ simultaneously. The larger the risk propensity parameters are, the smaller the risk index becomes. When the propensity parameters are $<1$, the risk propensity becomes a risk aversion, since the longitudinal and lateral risk indices are increased.

\begin{theorem}
\label{t1}
$r=0$ if and only if a collision is impossible to occur in the future given the current states of the vehicles (i.e. positions, velocities and accelerations) and following the braking and response time assumptions. 
\end{theorem}

\begin{proof}
$r=0 \iff r_{lon}=0 \lor r_{lat}=0 \iff (d^{lon} \geq d^{lon}_{min} > 0) \lor (d^{lat} \geq d^{lat}_{min} > 0)$, meaning at least one safety distance is being held.
\end{proof}

\begin{theorem}
\label{t2}
If $r=1$, avoiding a potential collision cannot be guaranteed under the assumptions of safe longitudinal and lateral distances, independent of the risk propensity parameters $\beta$ and $\gamma$.
\end{theorem}

\begin{proof}
$r=1 \iff r_{lon}=1 \land r_{lat}=1 \iff (\geq d^{lon} \geq d^{lon}_{min,brake}) \land (\geq d^{lat} \geq d^{lat}_{min,brake})$, meaning both the longitudinal and lateral safety distances under maximum braking are being violated. If the other vehicle brakes, a collision can not be guaranteed to be avoided under the model conditions.
\end{proof}

\section{Learning approach}

This section presents our proposed model for the safe and efficient passing of EMVs in autonomous traffic. A diagram can be found in Figure \ref{fig:diagram_method}. Most studies have addressed this problem based on rules and heavy simplifications of vehicle-to-vehicle interactions. Rule-based methods cannot cover all possible situations in autonomous traffic; therefore, we adopt a Multi-Agent Reinforcement learning approach. MARL allows us to simulate and train agents on thousands of different scenarios and achieve effective coordination among agents. In our formulation, all AVs and the EMV are represented as learnable agents that can take independent actions based on partial observability of the environment. We adopt the following assumptions and simplifications:

\begin{assumption}
\textbf{The environment considers a set of high-level discrete actions.}
In this study, we are interested in the multi-agent coordination policies, rather than the ability of individual vehicles to follow the road and stay on track. Thus, we use a discrete action space that considers high-level actions; accelerate, brake and change lane. This is a common assumption of multi-agent autonomous driving simulators \citep{zhou_smarts_2020, palanisamy_multi-agent_nodate, zhou_multi-agent_2022}, where vehicle dynamics are considered part of the environment. 

\end{assumption}

\begin{assumption}
\textbf{All agents have partial observability of the environment.}
Fixed-range AV sensors and the stochastic nature of vehicle-to-vehicle interactions cause AVs to have a limited perception of their surrounding environment \citep{lecerf_safer_2022}. We adopt a fixed range of 70m, which is realistic according to current communication protocols \cite{kenney_dedicated_2011}.
\end{assumption}

\subsection{Multi-Agent Deep Reinforcement Learning}

\begin{figure}[t!]
  \centering
  \includegraphics[scale=0.6]{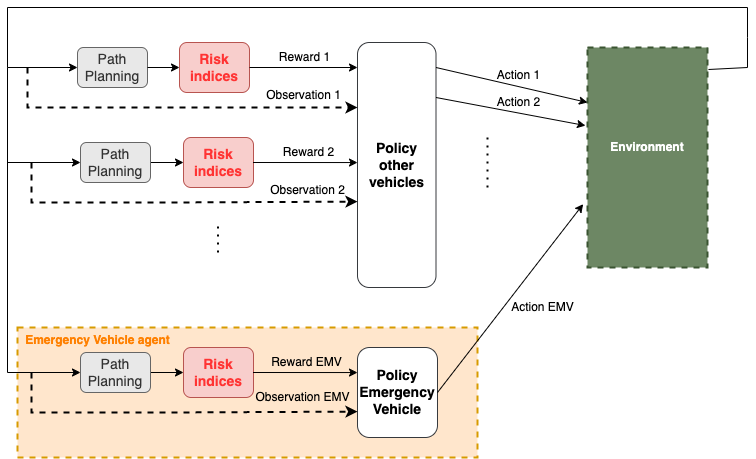}
  \caption{Diagram of the risk-aware learning methodology}
  \label{fig:diagram_method}
\end{figure}

The MARL problem statement can be formalised as a Markov game $G$, defined by the tuple $G = (N, \Omega, S, A, \mathcal{P}, \mathcal{R} , \gamma )$ where $N$ corresponds to the number of agents, $\Omega$ to the set of individual observations, $S$ to the set of states, $A$ to the joint action space, $P$ to the transition probabilities, $R$ to the reward functions and $\gamma$ to the discount factor.

We implement the Multi-Agent Policy Optimization algorithm (MAPPO), a state-of-the-art actor-critic MARL method that has shown to achieve superior performance in several coordination-heavy tasks \citep{yu_surprising_2021}. Proximal Policy Optimisation is a gradient-based method that approximates the policy directly. It includes a clipping objective $L^{Clip}$ that controls the policy update during the training period. This forces the policy update to be conservative if it deviates too much from the current policy. The actor and critic update equations for MAPOO are presented in the Appendix.

In contrast to the single-agent version, in MAPPO, all agents share a centralised critic during the training phase, but only the actors are needed in execution. Thus, it belongs to the Centralised Training Decentralized Execution (CTDE) class of algorithms. MAPPO is an on-policy method. Therefore, it does not rely on previous experience. This can make it less sample efficient than off-policy methods such as MADDPG \citep{lowe_multi-agent_2020} and QMIX \citep{rashid_qmix_2018}, at least in theory. In practice, MAPPO has been found to be a powerful and robust method in several cooperative environments, even outperforming state-of-the-art off-policy algorithms \citep{yu_surprising_2021}. The implemented MAPPO algorithm is presented in Algorithm \ref{alg:cap}.

\begin{algorithm}
\caption{MAPPO for autonomous traffic control
}\label{alg:cap}
\begin{algorithmic}
\State Initialize $\theta$, the parameters for the critic $V$ and the parameters $\phi$ for the actor $\pi$.
\For {$ep=1$ to $M$}
\State Retrieve initial state $s_{0}, \Omega_{0}$: Individual speeds, positions, lanes and perception of surrounding vehicles.
\For {$t=1$ to $T$}
\For {$j$ in $agents$}
\State Choose high-level action $a_{t}^{(j)} = \pi (\Omega_{t}^{(j)}; \theta)$
\State Calculate value $v^{(j)}=V(s_{t}^{(j)}, \phi)$
\EndFor 
\State Execute high-level actions and calculate new pose based on vehicle dynamics.
\State Observe $r_{t}, s_{t+1}, \Omega_{t+1}$.
\State Store $[s_{t}, \Omega_{t}, a_{t}, r_{t}, s_{t+1}, \Omega_{t+1}]$.

\State Compute advantage estimate $\hat{A}$ via GAE using the stored transitions.
\State Compute reward-to-go $\hat{R}$ on the stored transitions. 
\EndFor
\State Update $\theta$ using Adam optimiser on mini-batch
\State Update $\phi$ using Adam optimiser on mini-batch
\EndFor

\end{algorithmic}
\end{algorithm}

A graphic representation of the MAPPO approach for autonomous traffic is presented in figure \ref{fig:approach}. In the first instance, we consider a single critic for all agents during the training phase. A separate actor network allows agents to take individual actions during execution, without the need for a centralised state. Therefore the implementation of the proposed method results in decentralised multi-agent policies. We also explore using a separate critic network for each agent in a competitive version of the environment.

We represent the EMV as an actor taking independent actions in the environment, although we acknowledge it can be an autonomous or a human-driven vehicle. It should be noted that given the challenging and critical nature of emergency driving, it is not expected to have reliable autonomous EMVs in the coming years. Nevertheless, improving the safety of AV systems in the presence of EMVs is a crucial first step in the reliable deployment of AVs.

\subsection{Emergency Vehicle Multi-Agent Environment}

We use a bespoke Autonomous Traffic gym environment consisting of several AVs and one EMV driving through a congested road. In contrast to other gym environments modelling Autonomous Traffic, this environment considers all vehicles to be trainable agents with partial observability of their surroundings. Figure \ref{fig:environment} shows a representation of the environment. The environment was designed to ensure the safe and efficient transit of EMVs. To ensure this, we incorporate the risk metrics mentioned in the previous subsection, and performance metrics related to the speed of each vehicle.
All AVs and the EMV can take independent actions and receive an individual reward on each step. The reward function was designed to include the speed and the risk metrics of the EMV in order to guarantee cooperation between agents for prioritising the EMV.

\begin{figure}[t!]
  \centering
  \includegraphics{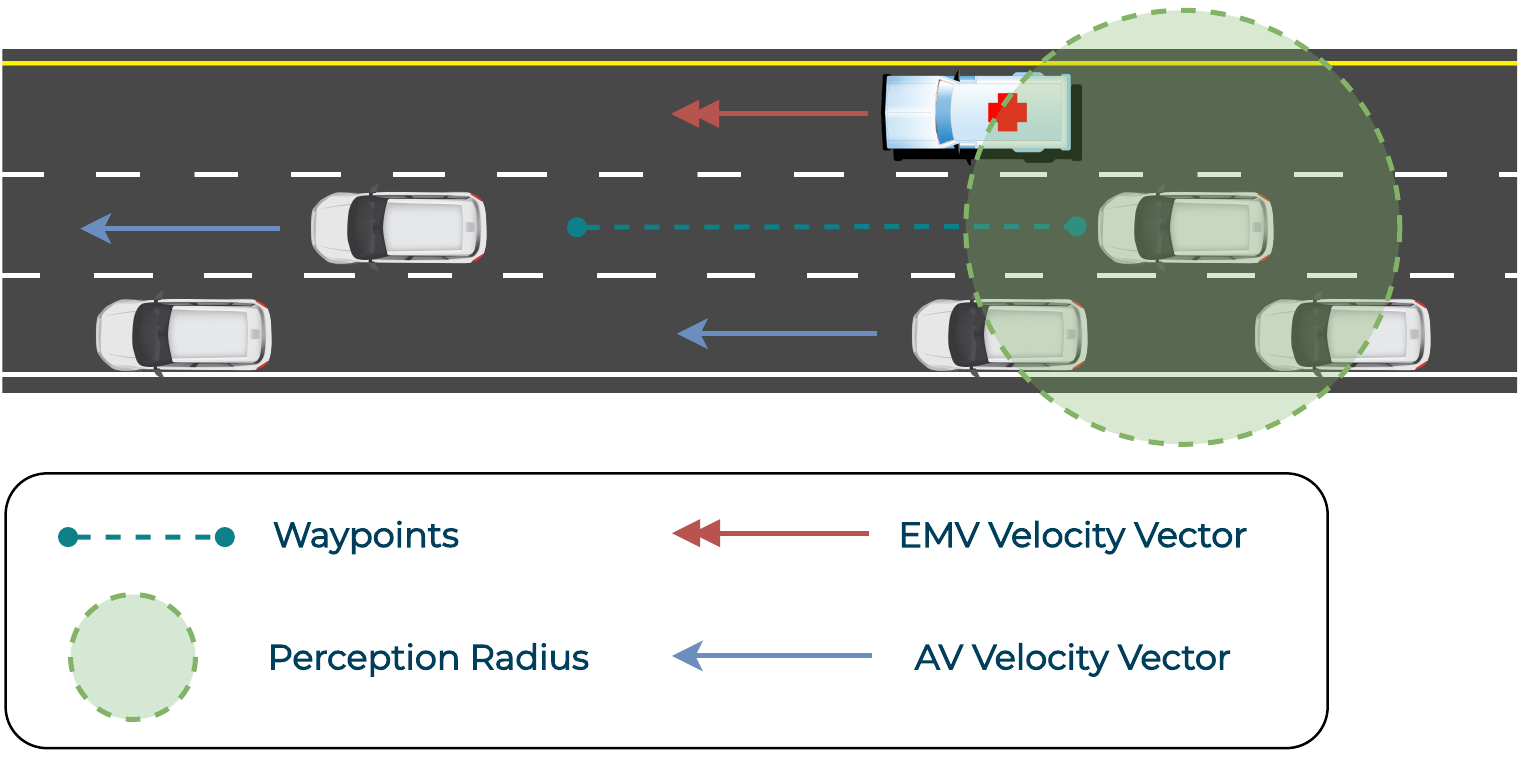}
  \caption{Representation of the multi-agent environment}
  \label{fig:environment}
\end{figure}

We define the individual elements of the environment as follows:

\begin{itemize}
    \item Agents $N$: There are two types of agents; AVs and the EMV. The AVs all have the same physical characteristics, range of speeds and perception range. The EMV has different physical characteristics and is allowed to go at higher speeds than the AVs.
    \item Observation $\Omega$: The local observation of the AVs consists of their own speed, position, lane and a limited perception of surrounding vehicles. Thus, AVs can recognise when an EMV is approaching at a certain distance but cannot see its global position. 
    \item Action $A$: All agents have a discrete action space with seven possible high-level decisions---accelerate; brake; heavy acceleration; heavy brake; change lane to the left; change lane to the right; keep previous velocity (no acceleration/braking). If an agent tries to perform a lane change but cannot for physical reasons, the action does nothing, and agents maintain their lane and speed.
    \item Reward $\mathcal{R}$: The reward function  is defined by the following equation:
    
    \begin{eqnarray}
	\mathcal{R} = r + v^{i} + p_{col} + p_{lcm} + v^{EV} + p_{lcm}^{EV}
    \end{eqnarray}
    
    Where r is agent's risk of collision as defined by equation (7), $v^{i}$ is the current speed measured in $m/s$, $p_{col}$ is a collision penalty and $p_{lcm}$ is a penalty for lane changing maneuvers.  In addition, $v^{EV}$ is EMV's current speed and $p_{lcm}^{EV}$ corresponds to an additional penalty for lane-changing manoeuvres. These last two components of the reward function guarantee cooperation among agents in the presence of the EMV. We normalised the individual components of the reward function between $0$ and $1$, except for the collision penalty, which we set to a very high value ($-100$) because of their relatively low frequency.
    
\end{itemize}

\section{Analysis}

The proposed MARL method was tested on a long track, where agents can go in all lanes in a single direction. We test two different scenarios, a two-lane road and a four-lane highway. However, the environment can be easily adapted to a multiple-lane setting with traffic going in both directions. The EMV has to make its way through the traffic by either overtaking or by AVs moving to another lane. First, we evaluate the performance of the proposed method, comparing it against two Car Following Models. Second, we consider different variations of the reward function and evaluate the trade-off between collision risk and traffic efficiency. Third, we assess the scalability of the proposed method by evaluating scenarios with a different number of agents. Finally, we present a fully competitive case, where AVs have their own independent reward and neural networks.

We used the EMV Multi-Agent Environment to train all policies in simulation. We performed a preliminary hyperparameter tunning experiment considering an instance with ten agents (1 EMV and 9 AVs). The combination of the best parameters for the MAPPO are presented in table \ref{parameters_mappo}. The parameters for the risk indices were selected using the studies on acceleration \citep{bokare2017acceleration, bosetti2014human}, they are presented in table \ref{parameters_risk}. The MARL results were averaged over four different random seeds. The response time was set to the frame rate of the simulations, which was $10 Hz$. All experiments were conducted using a 10-core, 3.70Ghz Intel Core i9-10900X CPU and an Nvidia GTX 3090 GPU. Furthermore, the parameters for the EMV Multi-agent Environment are presented in table \ref{parameters_env}. The minimum velocity is set to $7m/s$, which 

\begin{table}[h]
\footnotesize
\centering
\caption{Environment parameters.}
\label{parameters_env}
\begin{tabular}{@{}cc@{}}
\toprule
Parameter   & Values 
\\ \toprule
AV Max. Speed &  $20 m/s$   \\
AV Min. Speed & $7 m/s$ \\
AV Dimensions & $2m \times 4m \times 1.5m$\\
EMV Max. Speed &  $30 m/s$ \\
EMV Min. Speed & $7 m/s$ \\
Perception radius & $20 m$\\
Road loop length & $400 m$\\
\bottomrule
\end{tabular}
\end{table}

\begin{table}[h]
        \footnotesize
        \centering
        \caption{Risk indices parameters}
        \label{parameters_risk}
        \begin{tabular}{@{}cc@{}}
        \toprule
        Parameter   & Values 
        \\ \toprule
        Response time $\rho$ & $0.1s$  \\
        $a_{max}$ & $2.5 m/s^{2}$ \\
        $b_{max}$  &  $2.5 m/s^{2}$ \\
        $b_{min}$ & $1.0 m/s^{2}$ \\
        $B$ & $3.0 m/s^{2}$ \\
        $a^{lat}_{max}$ & $1 m/s^{2}$ \\
        $b^{lat}_{min}$ & $2.5 m/s^{2}$ \\
        $B^{lat}$ & $4.0 m/s^{2}$ \\
        $\beta$ & 1 \\
        $\gamma$ & 1 \\
        \bottomrule
        \end{tabular}
\end{table}

\begin{table}[h]
        \footnotesize
        \centering
        \caption{MAPPO parameters.}
        \label{parameters_mappo}
        \begin{tabular}{@{}cc@{}}
        \toprule
        Parameter   & Values 
        \\ \toprule
        N° episodes for training & $2000$   \\
        N° steps per episode &  $400$ \\
        Learning Rate & $5 \times 10^{-4}$ \\
        Critic Learning Rate & $5 \times 10^{-4}$ \\
        PPO epochs & $15$ \\
        Actor Network &  $MLP(64, 64, 7)$  \\       
        Critic Network & $MLP(64, 64, 1)$\\
        \bottomrule
        \end{tabular}
\end{table}

\subsection{Comparison to Car Following baselines}

We compare the proposed MARL approach against two baselines: the Gipps Car Following model \citep{gipps_model_1986} and an Adaptive Cruise Control (ACC) method based on Model Predictive Control (MPC). 

\subsubsection{Gipps Car Following Model}
 
The Gipps model is a well-known car-following approach that corresponds to the safe-distance class of models. We chose this model because it considers the risk of potential collision between vehicles and is simple to calibrate \citep{ahmed_review_2021}. In this model, the actions to change lanes are governed by a fixed set of rules. The rules consider the risk of collision, the presence of vehicles in the other lane and the potential gain in the average speed of the decision. The main disadvantage of this model is that it does not consider vehicle perception  \citep{lazar_review_2016}. In addition, minor variations on the leader vehicle can heavily affect the followers' reaction.

Gipps model determines the speed of each vehicle at time $t$ by applying two constraints on the acceleration: the desired speed must not be surpassed, and the critical safe distance should be maintained. The acceleration and deceleration can be expressed by the following equations:

\begin{eqnarray}
	v^{acc}_{n}(t+T) = v_{n}(t) + 2.5\cdot a_{n}\cdot T (1-\frac{v_{n}(t)}{v^{d}_{n}}) \cdot \sqrt{0.025 + \frac{v_{n}(t)}{v^{d}_{n}}}
\end{eqnarray}

\begin{eqnarray}
	v^{decc}_{n}(t+T) = -T \cdot d_{n} + \sqrt{T^{2} \cdot d^{2}_n + d_{n} \{ 2x_{n-1} (t) - x_{n} (t) - (S_{n-1}) - T \cdot v_{n} (t) + \frac{v_{n-1}(t)^{2}}{d'_{n-1}}} 
\end{eqnarray}

 where $n$ and $n-1$ are the follower and leader, respectively, $T$ is the reaction time, $v_{n}(t)$ and $v_{n-1}(t)$ are the speeds of the follower and leader respectively at time t. $v^{d}_{n}$ and $a_{n}$ are the follower's maximum desired speed and maximum acceleration, respectively. $d_{n}$ and $d'_{n-1}$ are the most severe braking capability and $x_{n}(t)$ and $x_{n-1}(t)$ are the longitudinal position of the follower and leader at time $t$. $S_{n-1}$ is the distance between vehicles at a stop.
 
 The implemented version of the Gipps model uses safe following distances set according to the RSS model introduced in \cite{shalev2017formal}. The reaction time was set to $T = 0.7 s$.
 
\subsubsection{Adaptive Cruise Control based on MPC}

This subsection describes the ACC method based on MPC. The method was adapted from \cite{zhu_safe_2020} for a multi-agent setting and discrete action space. The model considers a kinematic point-mass model that can be described by the following equation:

\begin{eqnarray}
    x(t+1) = Ax(t) + Bu(t)
\end{eqnarray}

\noindent where $x(t)=(S_{n-1, n}(t), \Delta V_{n-1,n}(t), V_{n}(t))^T$, $u(t)=a(t)$ and

\begin{equation}
    A
    =
    \begin{bmatrix}
    1 & \Delta T & 0 \\
    0 & 1 & 0 \\
    0 & 0 & 1
    \end{bmatrix}
\end{equation}

\begin{equation}
    B
    =
    \begin{bmatrix}
    -0.5\Delta T^{2} \\
    -\Delta T^{2} \\
    \Delta
    \end{bmatrix}
\end{equation}

The objective of an ACC is to follow the leader vehicle (LV) at the desired distance $\tilde{S}_{n-1,n}$, considering the vehicle speed and the headway $t_{hw}$:

\begin{eqnarray}
    \tilde{S}_{n-1,n} = V_n t_{hw}
\end{eqnarray}

The mathematical formulation of the MPC model is presented below:

\begin{mini!}
    {a}{\sum_{t=0}^{N-1} [(\frac{S_{n-1,n}(t)-\tilde{S}_{n-1,n}(t)}{S_{max}})^{2}+(\frac{\Delta V_{n-1,n}}{\Delta V_{max}})+(\frac{jerk(t)}{jerk_{max}})^2]}{}{}
    \addConstraint {x(t+1)}{=Ax(t) + Bu(t)}
    \addConstraint{S_{n-1,n}}{> 0}
    \addConstraint{V_{n}}{> 0}
    \addConstraint{-3 m/s^{2}}{\leq a_{n} \leq 3m/s^{2}}
\end{mini!}

where N is the prediction horizon, $S_{max}$, $\Delta V_{max}$ and $jerk_{max}$ are constants. Constraint (14b) corresponds to the point-mass model equations. Constraint (14c) ensures that the relative distance between the leader and the following vehicle is always greater than 0. Constraints (14d) and (14e) state that the speed of the following can't be negative, and its acceleration should be within a predefined range, respectively. We adopt the same parameter values used in \cite{zhu_safe_2020}, that is, $t_{hw}=1.2s$, a $S_{max}=15m$, $\Delta V_{max} = 8 m/s$. In addition, we set the jerk term to be equal to zero because the MARL formulation does not consider it as part of the objective. 

\subsubsection{Results comparison}

Table \ref{gipps} presents the comparison results between the proposed MARL method, the Gipps lane-changing model and the MPC method. We used a fixed number of AVs for this set of experiments and trained the algorithm for 2000 episodes. The results shown in the table correspond to averages over 10 test episodes. The proposed method consistently overperforms both the Gipps and MPC methods. It produces policies with higher average speeds for the EMV while maintaining safer distances between all vehicles. The collision risk determined by the risk index is also considerably lower for the MARL method. 

\def\arraystretch{1}
\begin{table*}[h]
\centering
\caption{Comparison to the Car Following baselines.}
\label{gipps}
\resizebox{\textwidth}{!}{%
\begin{tabular}{@{}cccccccc@{}}
\toprule
Scenario & Method & Reward &  EMV Speed & AV Speed & Col. rate &  Col. risk & Safety Dist. (m) \\
 \bottomrule
 \multirow{3}{*}{Road}
  & Gipps & $-452.06 \pm 279.99$ & $19.76 \pm 0.36$ & $\mathbf{19.57 \pm 0.34}$  & $15.30 \pm 10.23$   & $0.306 \pm 0.047$ & $6.51 \pm 0.92$ \\
  & MPC & $401 \pm 17.05$ & $14.29 \pm 0.45$ & $14.09 \pm 0.22$ & $1.00 \pm 3.00$ & $\mathbf{0.16 \pm 0.41}$ & $16.71 \pm 4.01$\\
  & MAPPO & $\mathbf{653 \pm 67.11}$ & $\mathbf{22.67 \pm 1.35}$  &  $19.05 \pm 0.42$ & $\mathbf{0.0} \pm \mathbf{0.0} $  & $0.175 \pm 0.013$ & $\mathbf{17.73 \pm 3.60}$ \\

\hline  
 
\multirow{3}{*}{Highway}  
   & Gipps & $229.23 \pm 110.23$ & $19.58 \pm 0.46$ & $19.47 \pm 0.25$  & $5.12 \pm 2.00$   & $0.360 \pm 0.06$ & $17.06 \pm 1.15$ \\
  & MPC & $523.44 \pm 96.08$ & $16.29 \pm 0.36$ & $14.94 \pm 0.24$ & $0.0 \pm 0.0$ & $0.11 \pm 0.020$ & $19.26 \pm 3.65$\\
  & MAPPO & $\mathbf{877.28 \pm 53.45}$ & $\mathbf{29.98 \pm 0.95}$  &  $\mathbf{19.87 \pm 0.35}$ & $\mathbf{0.0} \pm \mathbf{0.0} $  & $\mathbf{0.03 \pm 0.0025}$ & $\mathbf{37.01 \pm 7.01}$ \\
  
 \bottomrule
\end{tabular}}
\end{table*}

In contrast to the Car Following Models, the proposed method allows for the efficient transit of the EMV in heavy traffic conditions. Thus, the EMV can achieve higher speeds while keeping safe distances from other vehicles. Figure \ref{fig:boxplot_speeds} shows the distribution of average speeds for the EMV and the AVs for the three methods. The speed of the AVs is slightly lower than the maximum allowable speed $(20 m/s)$, because AVs can effectively collaborate and slow down to facilitate the passage of the EMV. On the other hand, the Gipps model does not prioritise the passage of the EMV because it maximises the speeds of all vehicles. The MPC method achieves nearly no collisions with a low overall risk of collision; however, the average speeds of the vehicles are also considerably lower.

\begin{figure*}[htb]
    \centering
    \subfloat[]{\includegraphics[scale=0.35]{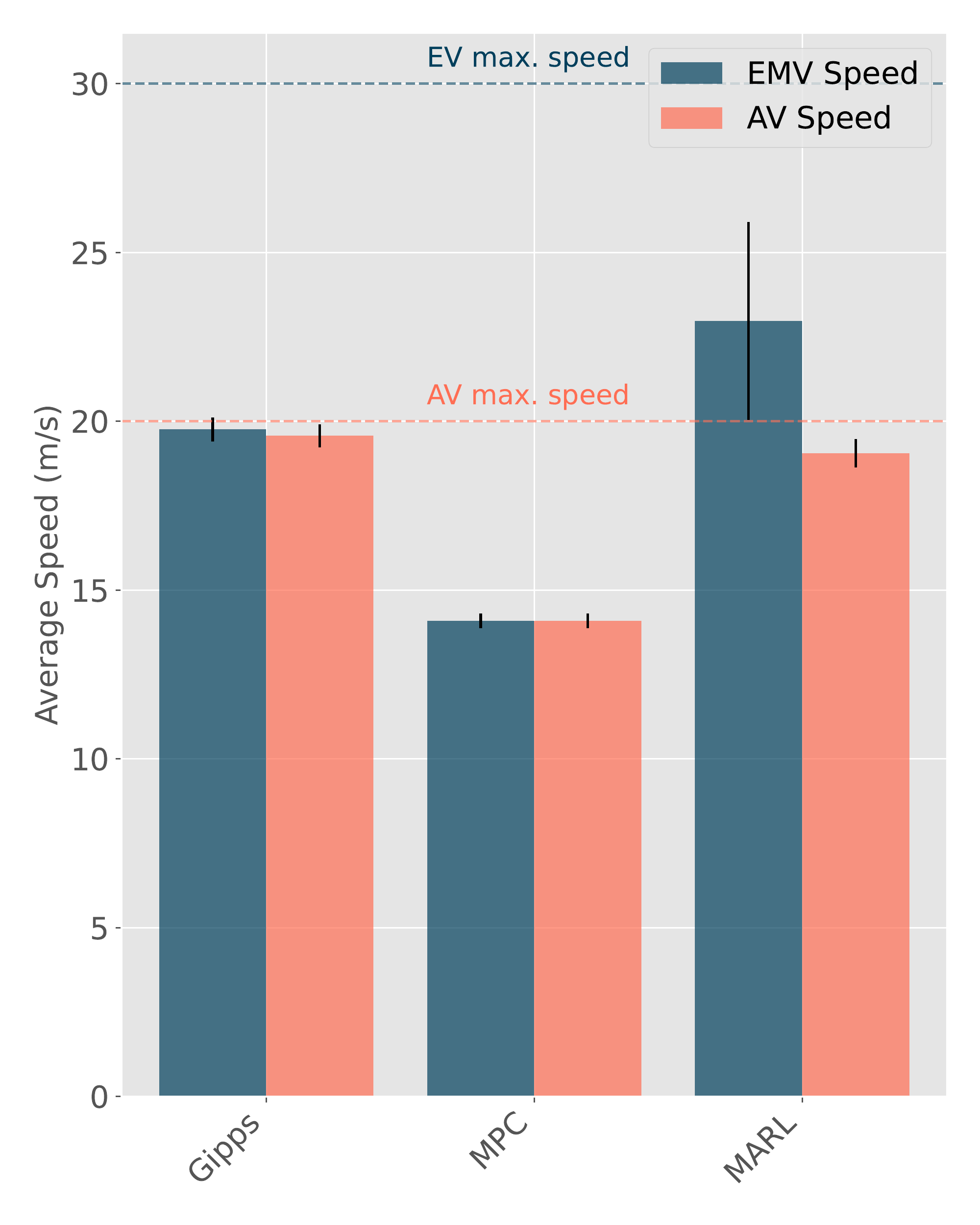} }
    \subfloat[]{\includegraphics[scale=0.35]{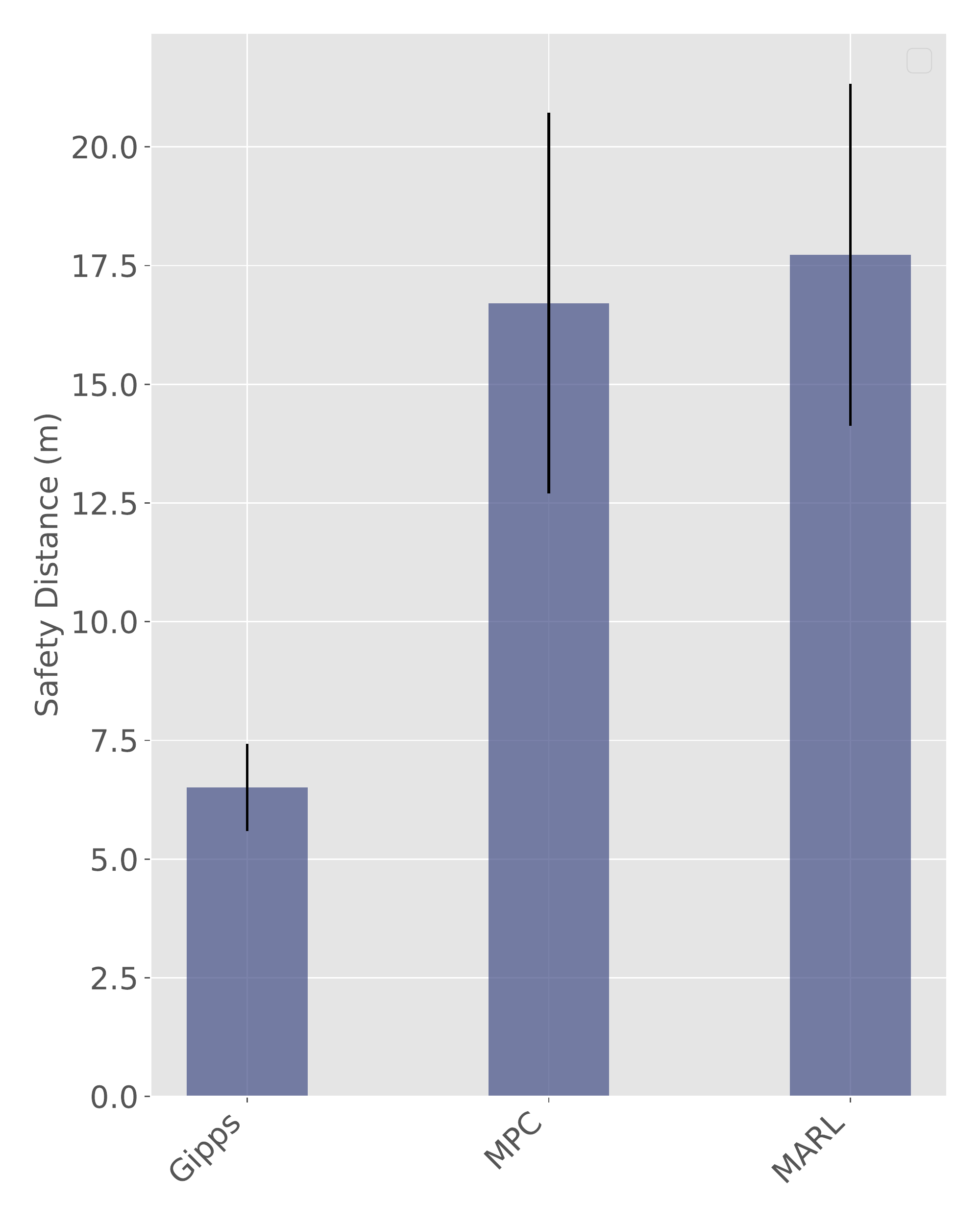}}
    \caption{Comparison of the policies obtained with the three different methods. (i) Average speeds for the EMV and AVs and (ii) Safety distances}
    \label{fig:boxplot_speeds}
\end{figure*}

The proposed method guarantees the safe transit of the EMV by incorporating the risk metric in the reward. No collisions occurred during the testing episodes, and the average collision risk is substantially lower than the Gipps method for both scenarios. The proposed method achieves a low collision risk thanks to preemptive manoeuvring actions from the AVs when the EMV is within their perception range. An example of these types of actions is shown in figure \ref{fig:lane_change}.

The MPC method achieves a low number of collisions and the lowest overall risk of collision. However, it is conservative in terms of speed, which causes it to perform worse in terms of total reward when compared to the MARL method.

\begin{figure}[t!]
  \centering
  \includegraphics[trim={0 1cm 0 0},clip, scale=0.7]{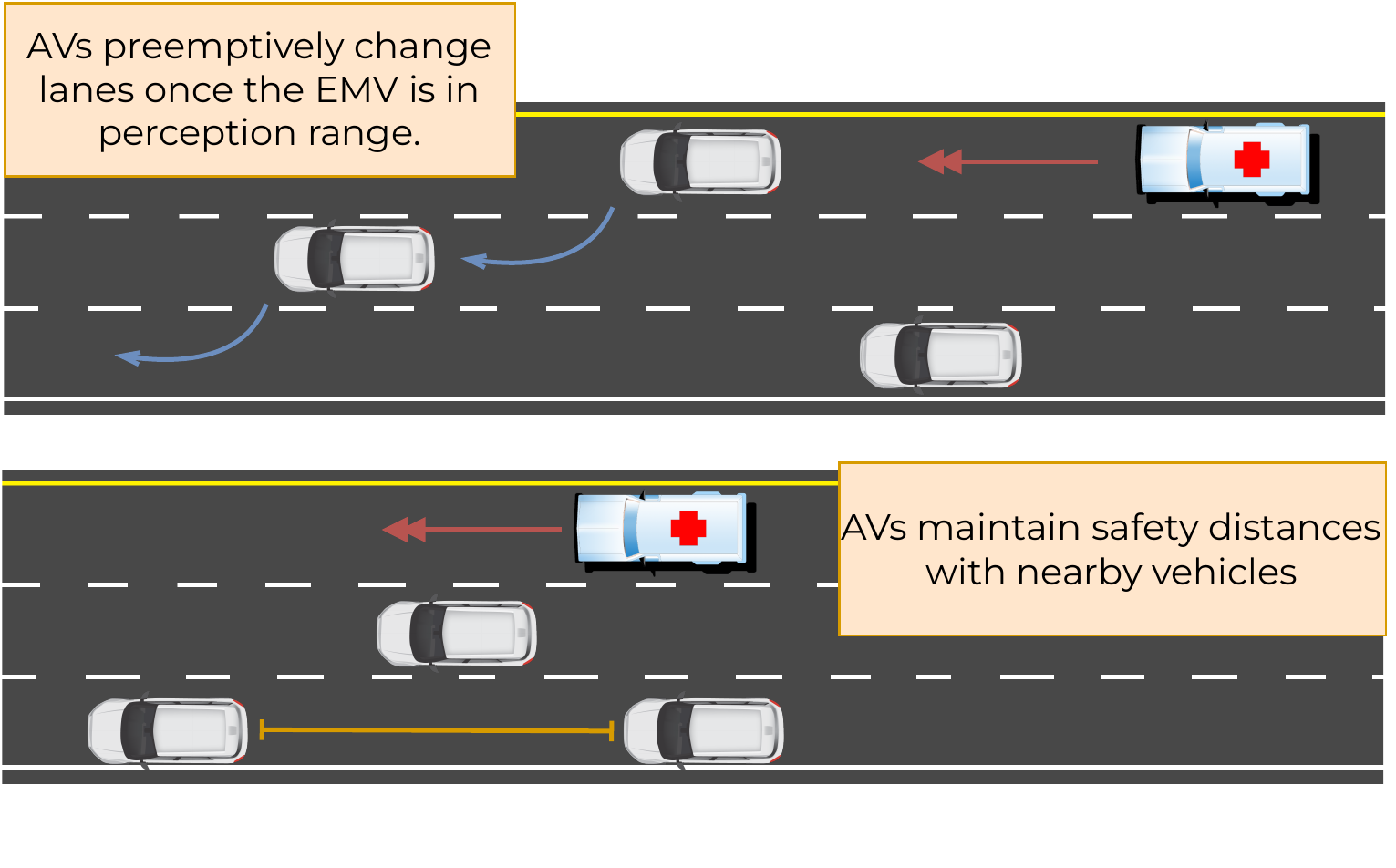}
  \caption{Example of preemptive lane-changing maneuvers by the AVs.}
  \label{fig:lane_change}
\end{figure}

\subsection{Reward shaping: Risk vs Efficiency}

We performed different experiments by varying the two main components of the reward function: risk of collision and traffic flow efficiency. The risk component is determined by equation (7), meanwhile the efficiency is determined by the average speed. If two vehicles are within the range of collision risk as defined by equation (3), then an increase in speed has a negative impact on the risk index. Conversely, if the speed decreases, the risk index improves. Thus, there is a trade-off between the speed and the risk of collision.

Table \ref{risk_efficiency} presents the results for different combinations of risk and efficiency within the reward function. We used a fixed number of agents (30) for these experiments and trained the MAPPO algorithm for 2000 episodes. The results shown in table \ref{risk_efficiency} correspond to averages over 10 test episodes. As expected, the speeds of both the EMV and the AVs are very high when the risk component is reduced. However, the collision rate and the collision risk are considerably worse compared to risk-aware policies (when the risk coefficient is greater than 0). In addition, the average safety distance between vehicles is lower when risk is not present in the reward function.

\def\arraystretch{1}
\begin{table*}[h!]
\centering
\caption{Results of different reward functions.}
\label{risk_efficiency}
\resizebox{\textwidth}{!}{%
\begin{tabular}{@{}cc|ccccc@{}}
\toprule
 Risk  & Efficiency &  EV Speed (m/s) & AV Speed (m/s) & Col. rate &  Col. risk & Safety Dist. (m) \\
 \bottomrule
  $0.0$ & $1.0$ & $27.83 \pm 1.16$ & $19.77 \pm 0.24$  & $3.90 \pm 1.40$   & $0.189 \pm 0.013$ & $16.95 \pm 3.82$ \\
  $0.5$ & $1.0$ & $18.86 \pm 0.81$  &  $17.86 \pm 0.16$ & $4.01 \pm 3.21$  & $0.176 \pm 0.021$ & $17.03 \pm 2.87$ \\
  $1.0$ & $1.0$ & $7.22 \pm 0.27$ & $7.22 \pm 0.11$ & $0.8 \pm 2.4$  & $0.114 \pm 0.031$ & $16.69 \pm 3.76$ \\
  $1.0$ & $0.5$ & $7.14 \pm 0.18$ &$7.12 \pm 0.16$ & $0.0 \pm 0.0$ & $0.105 \pm 0.015$  & $16.23 \pm 4.61$  \\
   $1.0$ & $0.0$ & $7.13 \pm 0.33$ &$7.09 \pm 0.055$ & $0.0 \pm 0.0$ & $0.094 \pm 0.019$  & $16.36 \pm 3.64$ \\
 
 \bottomrule
\end{tabular}}
\end{table*}

Figure \ref{fig:risk_objective} shows the effect of increasing the risk component in the reward function over the EMV average speed and the risk of collision. The speed of the EMV is reduced drastically as the risk component approaches the value of $1.0$. Moreover, when the risk component reaches the value $1.0$, the average speed is equal to the minimum velocity $(7 m/s)$. This is because the speed has a direct negative impact over the safe lateral and longitudinal safe distances, according to equations (1), (2), (4) and (5).

\begin{figure}[t!]
  \centering
  \includegraphics[scale=0.4]{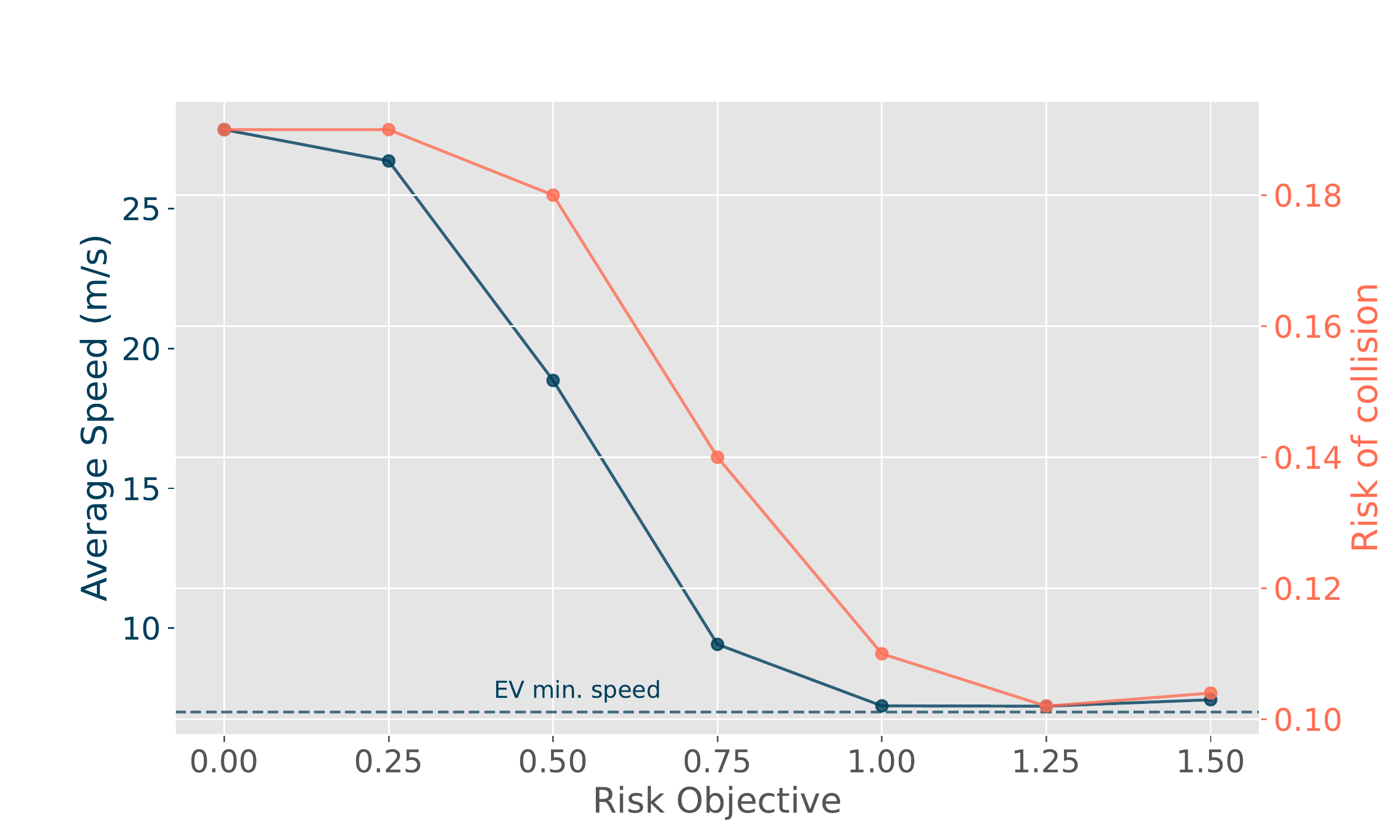}
  \caption{Influence of the risk component over the EMV speed and the risk of collision. The risk component is determined by equation (7).}
  \label{fig:risk_objective}
\end{figure}

\subsection{Scalability}

To test the scalability of the proposed method, we explored the effect of using a different number of agents over the training process and the performance of the trained policies. Additional agents do not alter the state dimension because agents have partial observability of the environment. Thus, a larger number of agents does not translate into higher neural network complexity. However, a larger number of agents produces higher environment complexity and, therefore, longer training time.

Figure \ref{fig:convergence_agents} shows the effect of different numbers of agents over the learning curves. The reward corresponds to the sum of rewards for all agents. For each case, four seeds were used, and the Figure shows the average and interval. Due to the higher number of vehicles on the road, the total average reward converges to considerably lower values when increasing the number of agents. 

\begin{figure}[t!]
  \centering
  \includegraphics[scale=0.5]{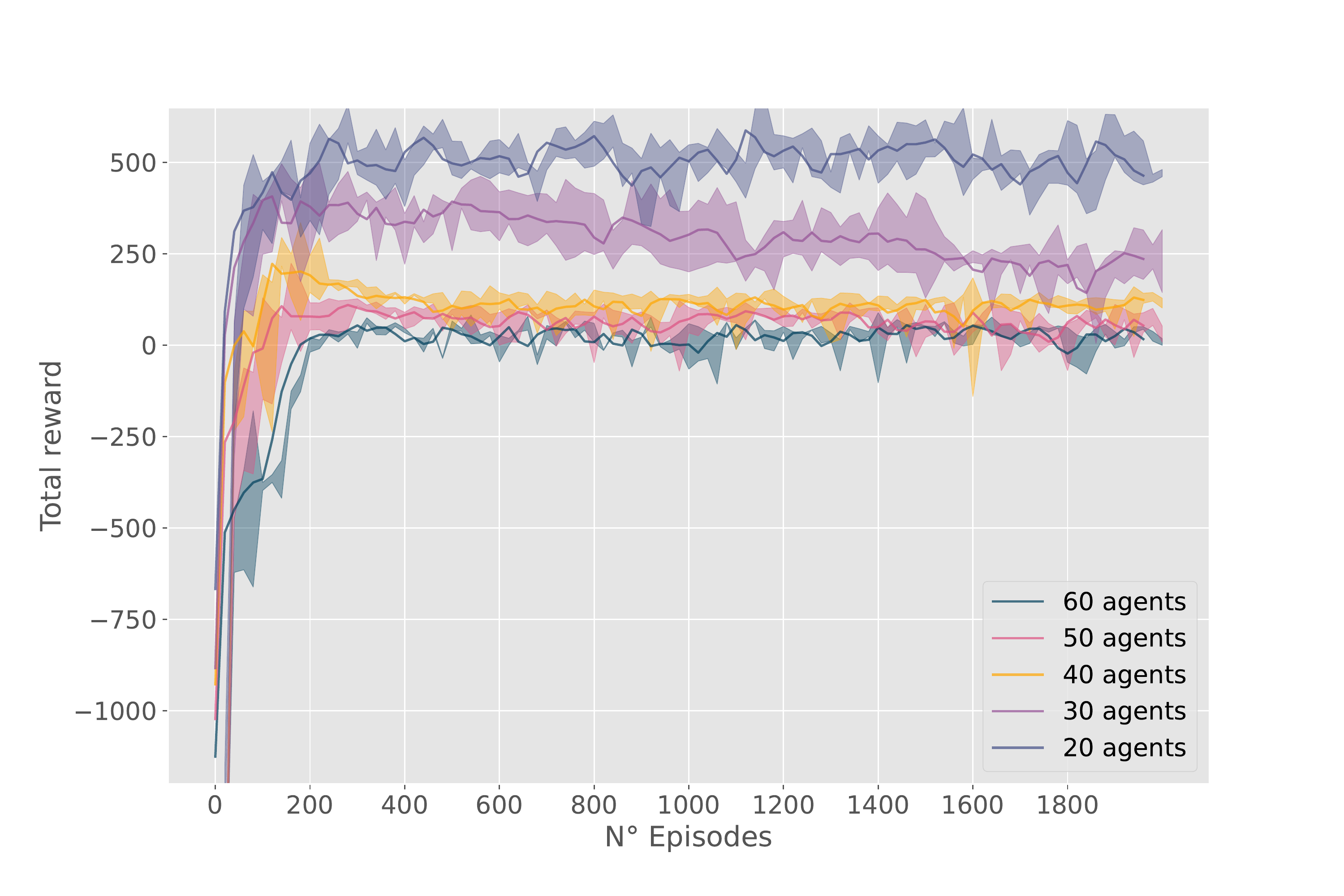}
  \caption{Learning curves for different number of agents.}
  \label{fig:convergence_agents}
\end{figure}

Table \ref{n_agents} shows the individual results for the trained multi-agent policies with a different number of agents. Results correspond to the average over 10 test episodes. All vehicles' initial positions and speeds were randomised for every test episode. When increasing the number of agents, AVs and the EMV must move at lower speeds, which causes the individual rewards to be lower. Moreover, vehicles will be closer to each other, which reduces the safety distance and thereby increases the risk of collision. The last column of the table contains the training time for the 2000 episodes, expressed as Frames Per Second (FPS). The total training time can be obtained by dividing the total number of steps by the FPS. For example, for the case with 20 agents, the total training time is $t = (2000 \cdot 400) /(31\cdot 3600) = 7.17 hrs$.

\def\arraystretch{1}
\begin{table*}[h!]
\centering
\caption{Results for different number of agents.}
\label{n_agents}
\resizebox{\textwidth}{!}{%
\begin{tabular}{@{}ccccccc@{}}
\toprule
 N° agents  & Reward &  EMV Speed & AV Speed & Average Dist. (m) &  Col. risk & FPS \\
 \bottomrule
  $20$ & $563.29 \pm 30.00$ & $26.89 \pm 0.49$ & $18.68 \pm 0.15$  & $17.03 \pm 3.87$   & $0.18 \pm 0.02$ & $31$ \\
  $30$ & $171.92\pm 36.64$ & $26.15 \pm 0.48$  &  $18.60 \pm 0.16$ & $13.97 \pm 3.64$  & $0.19 \pm 0.052$ & $21$ \\
  $40$ & $123.45 \pm 19.18$ & $18.30 \pm 1.08$ & $15.17 \pm 0.18$ & $8.67 \pm 5.54$  & $0.21 \pm 0.018$ & $14$ \\
  $50$ & $94.82 \pm 42.38$ & $16.27 \pm 2.37$ & $11.52 \pm 0.92$ & $8.70 \pm 4.38$ & $0.45 \pm 0.069$  & $10$  \\
  $60$ & $61.01 \pm 25.050$ & $11.37 \pm 0.99$ &  $10.11 \pm 0.58$ & $7.2 \pm 3.41$  & $0.50 \pm 0.028$ & $8$ \\
 
\bottomrule
\end{tabular}}
\end{table*}

 \subsection{Competitive scenario}

Until now, all experiments were carried out assuming a fully cooperative game, where all agents share the same networks (actor and critics), episode buffer and the reward function. We explored the performance of the proposed method in a competitive scenario, i.e., where agents seek to maximise their own individual rewards. For this, we created a separate actor and critic network for each agent and allowed them to maximise their own speed and penalise their collisions. Having a separate network for each agent makes the training substantially less efficient, increasing the training time nearly threefold. For this set of experiments, we used the two-lane road scenario and trained the MAPPO algorithm for 2000 episodes.

In both the cooperative and competitive scenarios, the MAPPO algorithm can guarantee no collisions and high safety distances between all vehicles. In addition, the resulting policies allow the EMV to go at a faster average speed. However, in the competitive scenario, vehicles do not coordinate to let the EMV pass efficiently. Figure \ref{fig:cooperative_com} contrasts the difference between cooperative and competitive multi-agent policies. In the fully cooperative case, vehicles move aside to minimise the braking and the number of lane changes of the EMV. In contrast, in the competitive scenario, the AVs do not change lanes preemptively, and they just maximise their own velocity; therefore, the EMV has to constantly overtake other vehicles. This causes the EMV to achieve a considerably lower average speed, which can be seen in figure \ref{fig:boxplot_coop} (i). On the other hand, figure \ref{fig:boxplot_coop} (ii) shows the difference in average risk of collision for the EMV. Due to the constant overtaking manoeuvres done by the EMV, the risk of collision of the substantially greater in the competitive case. The MAPPO algorithm can still generate collision-free policies in this case but the overall traffic efficiency is substantially lower.

\begin{figure}[t!]
  \centering
  \includegraphics[scale=0.7, trim={0 1.5cm 0 0},clip]{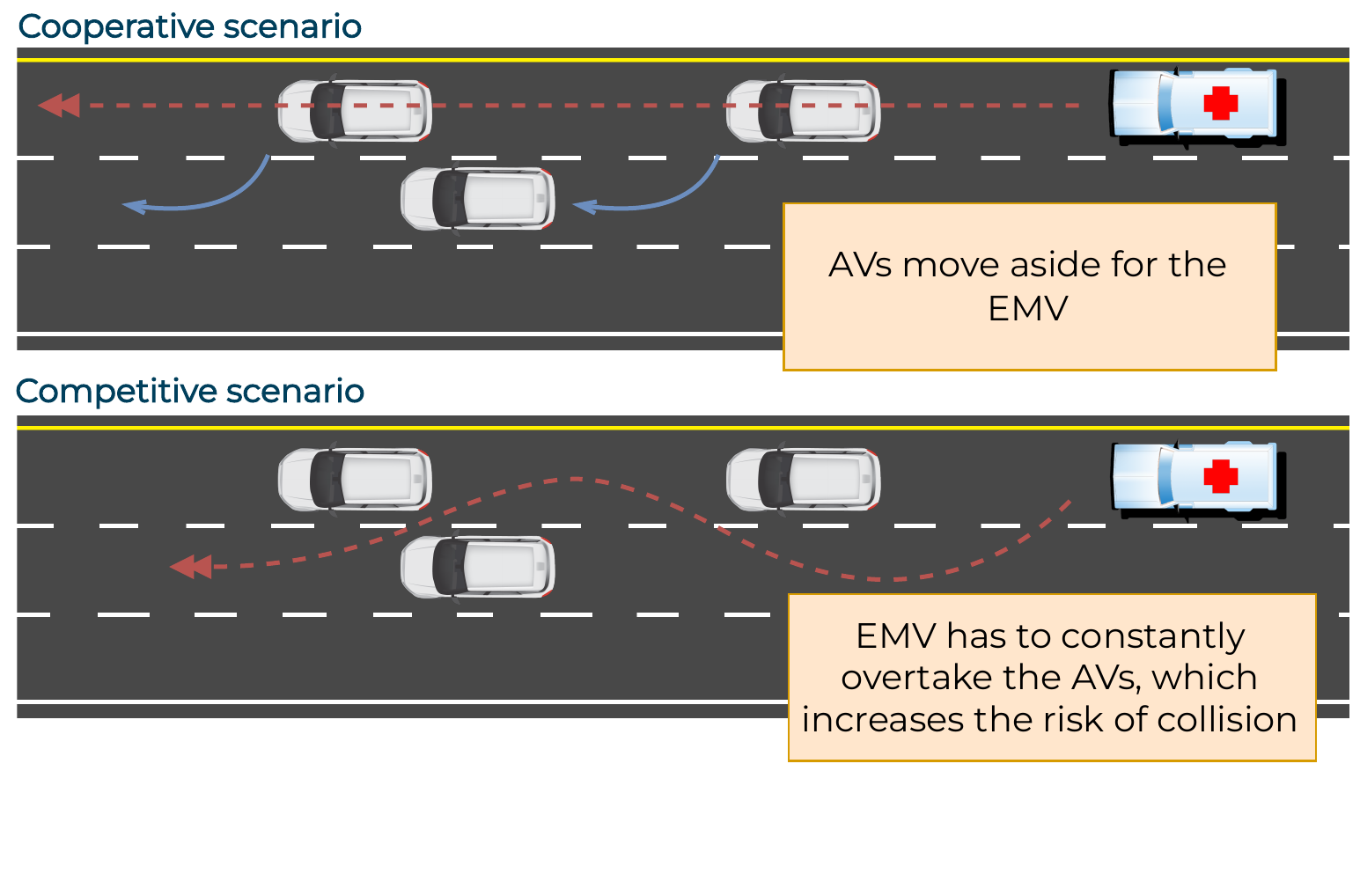}
  \caption{Cooperative vs. competitive multi-agent policies. In the fully cooperative scenario, the AVs move aside to let the EMV pass, while in the competitive scenario the EMV has to constantly perform overtaking manoeuvres.}
  \label{fig:cooperative_com}
\end{figure}

\begin{figure*}[htb]
    \centering
    \subfloat[]{\includegraphics[scale=0.35]{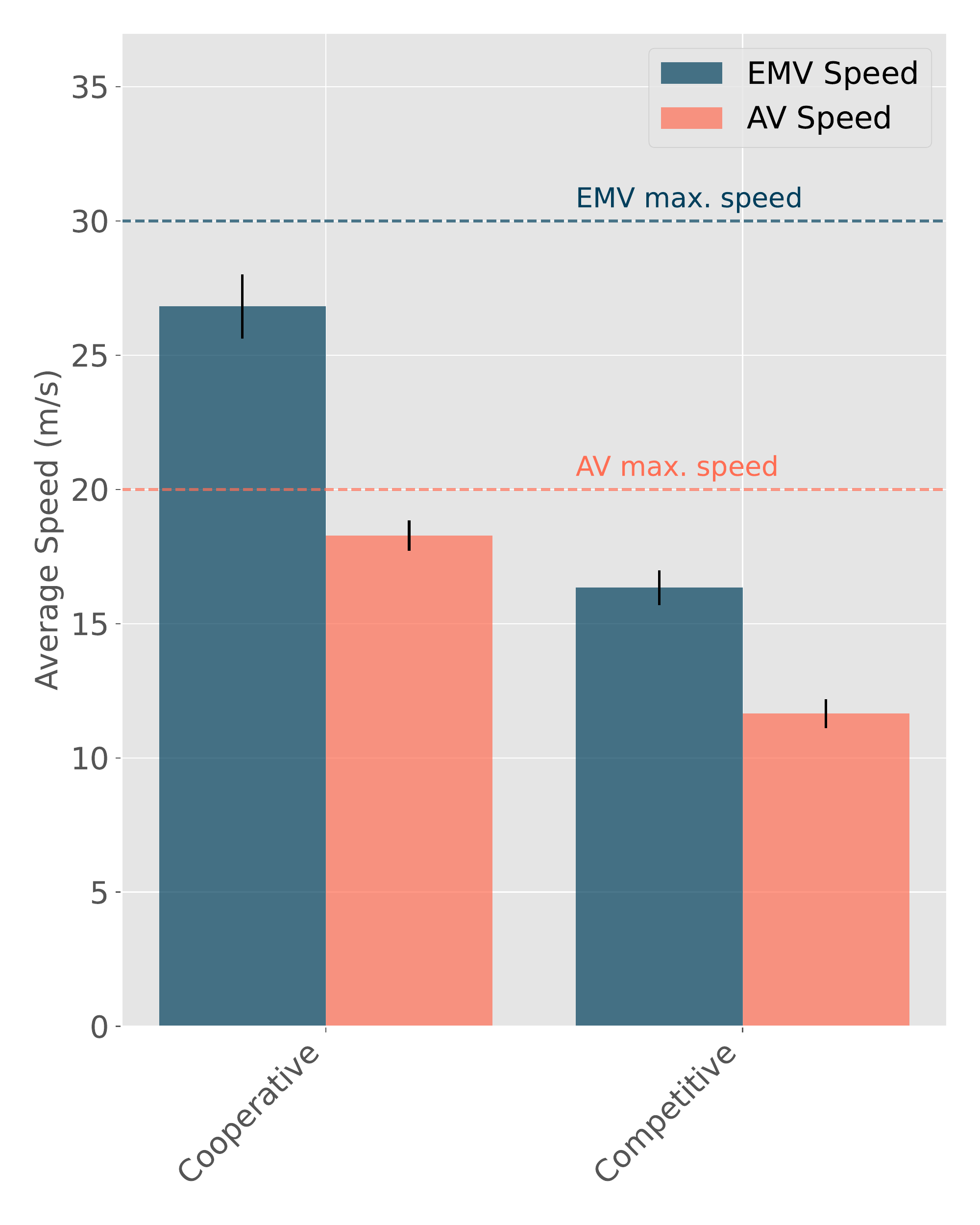}}
    \subfloat[]{\includegraphics[scale=0.35]{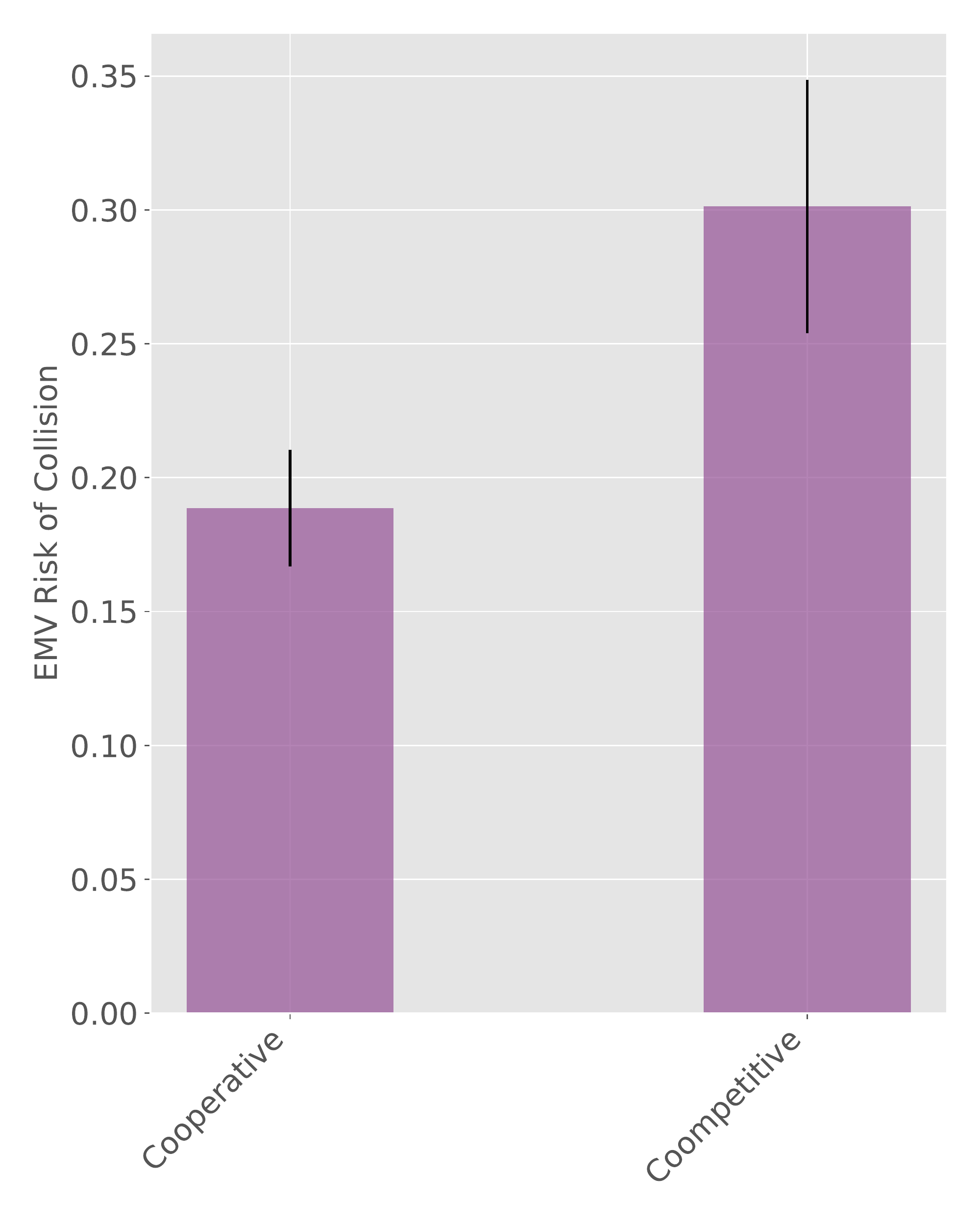} }
    \caption{Comparison between the cooperative and competitive scenarios (i) Average speeds (ii) EMV Risk of collision}
    \label{fig:boxplot_coop}
\end{figure*}

\section{Conclusion}

AVs have the potential to improve road safety and efficiency, although there are still many challenges regarding their practical application in traffic systems. This study proposed a multi-agent approach for the safe and efficient manoeuvring of autonomous traffic in the presence of EMVs. We used the MAPPO algorithm combined with risk metrics to determine manoeuvring actions for all cars in congested road settings. We explored the capabilities of the proposed approach in four case studies. First, we compared it with a traditional rule-based Car Following Model. Second, we contrasted the two main components of the reward function; risk and traffic efficiency. Third, we investigated the scalability and performance with different numbers of agents. Finally, we presented a competitive scenario where vehicles seek to maximise their own independent reward.

Results of the four case studies show the potential and capabilities of the proposed method. It consistently outperforms the car following baselines, obtaining higher average speeds for the EMV while maintaining a lower overall risk of collision. Furthermore, we show how the reward function can be tuned to achieve higher safety or traffic efficiency. However, a minimum amount of risk component is needed to guarantee safe transit and no collisions. A higher number of agents translates into lower average speeds of the EMV and AVs; however, the algorithm can still guarantee safe manoeuvring. The proposed method is also applicable to competitive scenarios, where vehicles act selfishly, maximising their own speed. The resulting policies involve constant manoeuvring of the EMV, which increases its risk of collision when compared to a fully cooperative case.

There is some future work that can explore new settings and can overcome some of the limitations of this study. We considered just a single road environment, with vehicles going in one direction. Other scenarios can be evaluated, such as a highway merging, two-directional lanes, intersections and the possibility of the EMV moving in the contraflow. In addition, explicit communication between agents to improve overall traffic safety should be explored in the future. However, having a one-to-one communication network between CAVs is not scalable; thus, innovative, efficient methods are needed for cases with multiple agents.

\section*{Acknowledgement(s)}
This research was partially supported by the President's Scholarship Programme funded by Imperial College London, and the Chilean National Agency for Research and Development (ANID) through the "BECAS DOCTORADO EN EXTRANJERO" program, Grant No. 72210279.

\bibliographystyle{apacite}
\bibliography{ambulance}

\section{Appendices}

In our formulation, the policy maps the agents' actions to a categorical distribution, which represents the different possible actions within the environment. During training, the actor network seeks to maximise the following equation \citep{yu_surprising_2021}:

\begin{equation}
    L(\theta) = [\frac{1}{Bn} \sum_{i=1}^{B} \sum_{k=1}^{n}  min(r_{\theta, i}^{(k)}A_{i}^{(k)}, clip(r_{\theta, i}^{(k)}, 1-\epsilon, 1+ \epsilon)A_{i}^{(k)}] + \sigma \frac{1}{Bn} \sum_{i=1}^{B} \sum_{k=1}^{n} S [\pi_{\theta}(\Omega_{i}^{(k)})],
\end{equation}

,where
\begin{equation}
    r_{\theta, i}^{(k)} = \frac{\pi_{\theta}(a_{i}^{(k)}| o_{i}^{(k)})}{\pi_{\theta_{old}}(a_{i}^{(k)}| o_{i}^{(k)})},
\end{equation}

where $A_{i}^{(k)}$is calculated using Generalized Advantage Estimation (GAE) with advantage normalization, $S$ is the policy entropy and $\sigma$ is the entropy coefficient. The critic minimises the following equation:

\begin{equation}
    L(\phi) = \frac{1}{Bn} \sum_{i=1}^{B} \sum_{k=1}^{n}(max[V_{\phi}(s_{i}^{(k)}) - \hat{R_{i}})^{2}, (clip(V_{\phi}(s_{i}^{(k)}), V_{\phi_{old}}(s_{i}^{(k)})-\epsilon, V_{\phi_{old}}(s_{i}^{(k)})+\epsilon)- \hat{R_{i}})^{2}],
\end{equation}

where $\hat{R_{i}}$ is the discounted error to go, $B$ is the batch size and $n$ the number of agents.

\appendix

\end{document}